\documentclass[11pt]{article}
\usepackage{natbib,fullpage}

\usepackage[utf8]{inputenc} \usepackage[T1]{fontenc}    \usepackage{hyperref}       \usepackage{url}            \usepackage{booktabs}       \usepackage{amsfonts}       \usepackage{nicefrac}       \usepackage{microtype}      \usepackage{enumerate}

\usepackage{amssymb}
\usepackage{pifont}

\usepackage{amsmath,amsfonts,bm}

\def\eqref#1{equation~\ref{#1}}

\def\1{\bm{1}}

\def\eps{{\epsilon}}

\DeclareMathAlphabet{\mathsfit}{\encodingdefault}{\sfdefault}{m}{sl}
\SetMathAlphabet{\mathsfit}{bold}{\encodingdefault}{\sfdefault}{bx}{n}

\def\gE{{\mathcal{E}}}

\newcommand{\EE}{\mathbb{E}}

\DeclareMathOperator*{\argmax}{arg\,max}

\newcommand{\wt}[1]{\widetilde{#1}}
 \usepackage{hyperref}
\usepackage{url}

\usepackage[utf8]{inputenc} \usepackage{hyperref}       \usepackage{url}            \usepackage{booktabs}       \usepackage{amsfonts}       \usepackage{nicefrac}       \usepackage{microtype}      \usepackage{natbib}
\usepackage{enumerate}
\usepackage{enumitem}
\usepackage{hhline}
\usepackage{makecell}

\usepackage{graphicx} \usepackage{caption}
\usepackage{subcaption}
\usepackage{amsmath}
\usepackage{amsthm}
\usepackage{amssymb}
\usepackage{tikz}
\usepackage{xcolor}
\usetikzlibrary{arrows}

\allowdisplaybreaks

\usepackage{mathrsfs}

\usepackage{algorithm}
\usepackage[noend]{algpseudocode}
\usepackage{hyperref}
\usepackage{bm,todonotes}

\allowdisplaybreaks

\newcommand{\RR}{\mathbb{R}}

\newcommand{\poly}{\mathrm{poly}}

\def\cA{\mathcal{A}}

\def\cE{\mathcal{E}}

\def\cS{\mathcal{S}}

\def\cS{\mathcal{S}}

\newcommand{\norm}[1]{\left\|#1\right\|}

\newcommand{\abs}[1]{\left|#1\right|}
\newcommand{\expect}{\mathbb{E}}

\newcommand{\states}{\mathcal{S}}
\newcommand{\trans}{P}

\newcommand{\actions}{\mathcal{A}}

\newtheorem{thm}{Theorem}[section]
\newtheorem{lem}{Lemma}[section]

\newtheorem{asmp}{Assumption}[section]

\newcommand{\mdp}{\mathcal{M}}

\def\approxcorrect{\cmark\kern-1.4ex\raisebox{.30ex}{$\xmark$}}

\newcommand{\idxn}[1][]{\ifthenelse{\equal{#1}{}}{\mathsf{INDQ}_n}{\mathsf{INDQ}_{#1}}}

\title{
On Reward-Free Reinforcement Learning with \\
Linear Function Approximation
}

\author{
Ruosong Wang\\Carnegie Mellon University\\\texttt{ruosongw@andrew.cmu.edu} \and Simon S. Du\\ University of Washington, Seattle \\ \& Institute for Advanced Study\\ \texttt{ssdu@uw.edu} \and Lin F. Yang \\ University of California, Los Angles \\ \texttt{linyang@ee.ucla.edu} \and Ruslan Salakhutdinov \\Carnegie Mellon University \\ \texttt{rsalakhu@cs.cmu.edu}
}
\date{}

\begin{document}
\maketitle

\begin{abstract}
	Reward-free reinforcement learning (RL) is a framework which is suitable for both the batch RL setting and the setting where there are many reward functions of interest. During the exploration phase, an agent collects samples without using a pre-specified reward function. After the exploration phase, a reward function is given, and the agent uses samples collected during the exploration phase to compute a near-optimal policy. Jin et al. [2020] showed that in the tabular setting, the agent only needs to collect polynomial number of samples (in terms of the number states, the number of actions, and the planning horizon) for reward-free RL. However, in practice, the number of states and actions can be large, and thus function approximation schemes are required for generalization. In this work, we give both positive and negative results for reward-free RL with linear function approximation. We give an algorithm for reward-free RL in the linear Markov decision process setting where both the transition and the reward admit linear representations. The sample complexity of our algorithm is polynomial in the feature dimension and the planning horizon, and is completely independent of the number of states and actions. We further give an exponential lower bound for reward-free RL in the setting where only the optimal $Q$-function admits a linear representation. Our results imply several interesting exponential separations on the sample complexity of reward-free RL. 
 \end{abstract}

\section{Introduction}
\label{sec:intro}
In reinforcement learning (RL), an agent repeatedly interacts with an unknown environment to maximize the cumulative reward.
To achieve this goal, RL algorithm must be equipped with exploration mechanisms to effectively solve tasks with long horizons and sparse reward signals. 
Empirically, there is a host of successes by combining deep RL methods with different exploration strategies.
However, the theoretical understanding of exploration in RL by far is rather limited. 

In this work we study the reward-free RL setting which was formalized in the recent work by \citet{jin2020reward}.
There are two phases in the reward-free setting: the exploration phase and the planning phase.
During the exploration phase, the agent collects trajectories from an unknown environment without any pre-specified reward function.
Then, in the planning phase, a specific reward function is given to the agent, and the goal is to use samples collected during the exploration phase to output a near-optimal policy for the given reward function.
From a practical point of view, this paradigm is particularly suitable for 1) the batch RL setting~\citep{bertsekas1996neuro} where data collection and planning are explicitly separated and 2) the setting where there are multiple reward function of interest, e.g., constrained RL~\citep{achiam2017constrained,altman1999constrained,miryoosefi2019reinforcement,tessler2018reward}.
From a theoretical point view, this setting separates the exploration problem and the planning problem which allows one to handle them in a theoretically principled way, in contrast to the standard RL setting where one needs to deal both problems simultaneously.

Key in this framework is to collect a dataset with sufficiently good coverage over the state space during the exploration phase, so that one can apply a batch RL algorithm on the dataset~\citep{chen2019information,agarwal2019optimality,antos2008learning,munos2008finite} during the planning phase.
For the reward-free RL setting, existing theoretical works only apply to the tabular RL setting.
\citet{jin2020reward}~showed that in the tabular setting where the state space has bounded size, $\widetilde{O}(\mathrm{poly}(|\states||\actions|H) / \varepsilon^2)$ samples during the exploration phase is \emph{necessary and sufficient} in order to output $\varepsilon$-optimal policies in the planning phase. Here, $\abs{\states}$ is the number of states, $\abs{\actions}$ is the number of actions and $H$ is the planning horizon.

The sample complexity bound in~\citep{jin2020reward}, although being near-optimal in the tabular setting, can be unacceptably large in practice due to the polynomial dependency on the size of the state space.
For environments with a large state space, function approximation schemes are needed for generalization.
RL with linear function approximation is arguably the simplest yet most fundamental setting.
Clearly, in order to understand more general function classes, e.g., deep neural networks, one must understand the class of linear functions first.
In this paper, we study RL with linear function approximation in the reward-free setting, and our goal is to answer the following question:
\begin{center}
\emph{
Is it possible to design provably efficient RL algorithms with linear function approximation in the reward-free setting?
}
\end{center}

We obtain both polynomial upper bound and hardness result to the above question. 
 
\paragraph{Our Contributions.}
Our first contribution is a provably efficient algorithm for reward-free RL under the linear MDP assumption~\citep{yang2019sample,jin2019provably}, which, roughly speaking, requires both the transition operators and the reward functions to be linear functions of a $d$-dimensional feature extractor given to the agent.
See Assumption~\ref{asmp:lin_mdp} for the formal statement of the linear MDP assumption.
Our algorithm, formally presented in Section~\ref{sec:algo}, samples $\widetilde{O}\left(d^3H^6/\varepsilon^2\right)$ trajectories during the exploration phase, and outputs $\varepsilon$-optimal policies for an arbitrary number of reward functions satisfying Assumption~\ref{asmp:lin_mdp} during the planning phase with high probability.
Here $d$ is the feature dimension, $H$ is the planning horizon and $\varepsilon$ is the required accuracy. 

One may wonder whether is possible to further weaken the linear MDP assumption, since it requires the feature extractor to encode model information, and such feature extractor might be hard to construct in practice. 
Our second contribution is a hardness result for reward-free RL under the linear $Q^*$ assumption, which only requires the optimal value function to be a linear function of the given feature extractor and thus weaker than the linear MDP assumption.
Our hardness result, formally presented in Section~\ref{sec:hardness}, shows that under the linear $Q^*$ assumption, any algorithm requires exponential number of samples during the exploration phase, so that the agent could output a near-optimal policy during the planning phase with high probability.
The hardness result holds even when the MDP is deterministic. 

Our results highlight the following conceptual insights.
\begin{itemize}
\item \textbf{Reward-free RL might require the feature to encode model information.} Under model-based assumption (linear MDP assumption), there exists a polynomial sample complexity upper bound for reward-free RL, while under value-based assumption (linear $Q^*$ assumption), there is an exponential sample complexity lower bound.
Therefore, the linear $Q^*$ assumption is {\em strictly weaker} than the linear MDP assumption in the reward-free setting. 
\item \textbf{Reward-free RL could be exponentially harder than standard RL.}  For deterministic systems, under the assumption that the optimal $Q$-function is linear, there exists a polynomial sample complexity upper bound~\citep{wen2013efficient} in the standard RL setting. However, our hardness result demonstrates that under the same assumption, any algorithm requires exponential number of samples in the reward-free setting. 
\item \textbf{Simulators could be exponentially more powerful.} 
In the setting where the agent has sampling access to a generative model (a.k.a. simulator) of the MDP, the agent can query the next state $s'$ sampled from the transition operator given any state-action pair as input.
In the supplementary material, we show that for deterministic systems, under the linear $Q^*$ assumption, there exists a polynomial sample complexity upper bound in the reward-free setting when the agent has sampling access to a generative model.
Compared with the hardness result above, this upper bound demonstrates an exponential separation between the sample complexity of reward-free RL in the generative model and that in the standard RL model. 
To the best our knowledge, this is the first exponential separation between the standard RL model and the generative model for a natural question.
\end{itemize}

 \subsection{Related Work}
\label{sec:rel}
Practitioners have proposed various exploration algorithms for RL without using explicit reward signals~\citep{oudeyer2007intrinsic,schmidhuber2010formal,bellemare2016unifying, houthooft2016vime, tang2017exploration,florensa2017automatic, pathak2017curiosity, tang2017exploration, achiam2017constrained,hazan2018provably,burda2018exploration,colas2018curious,co2018self,nair2018visual, eysenbach2018diversity,
pong2019skew}.
Theoretically, for the tabular case, while the reward-free setting is first formalized in \cite{jin2020reward}, algorithms in earlier works also guarantee to collect a polynomial-size dataset with coverage guarantees~\citep{brafman2002r,hazan2018provably,du2019decoding,misra2019kinematic}.\footnote{\citet{du2019decoding,misra2019kinematic} studied the rich-observation setting where the observations are generated from latent states. The latent state dynamics is a tabular one.}
\citet{jin2020reward} gave a new algorithm which has $\widetilde{O}(\abs{\states}^2\abs{\actions}\poly(H) / \varepsilon^2)$ sample complexity.
They also provided a lower bound showing the dependency of their algorithm on $\abs{\states}, \abs{\actions}$ and $\varepsilon$ is optimal up to logarithmic factors.
One of questions asked in \citep{jin2020reward} is whether their result can be generalized to the function approximation setting.

This paper studies linear function approximation.
Linear MDP is the setting where both the transition and the reward are linear functions of a given feature extractor. 
Recently, in the standard RL setting, many works~\citep{yang2019sample,jin2019provably,cai2019provably,zanette2019frequentist} have provided polynomial sample complexity guarantees for different algorithms in linear MDPs.
Technically, our algorithm, which works in the reward-free setting, combines the algorithmic framework in~\citep{jin2019provably} with a novel exploration-driven reward function (cf. Section~\ref{sec:algo}). 
Linear $Q^*$ is another setting where only the optimal $Q$-function is assumed to be linear, which is weaker than the assumptions in the linear MDP setting.
In the standard RL setting, it is an open problem whether one can use polynomial number of samples to find a near-optimal policy in the linear $Q^*$ setting~\citep{Du2020Is}.
Existing upper bounds all require additional assumptions, such as (nearly) deterministic transition~\cite{wen2013efficient,du2019provably,du2020agnostic}.

\section{Preliminaries}
\label{sec:pre}
Throughout this paper, for a given positive integer $N$, we use $[N]$ to denote the set $\{1, 2, \ldots, N\}$.
\subsection{Episodic Reinforcement Learning}
\label{sec:mdp}

Let $\mdp =\left(\states, \actions, P ,r, H, \mu\right)$ be a \emph{Markov decision process} (MDP)
where $\states$ is the state space, 
$\actions$ is the action space with bounded size, 
$P = \{P_h\}_{h = 1}^H$ where 
$P_h: \states \times \actions \rightarrow \Delta\left(\states\right)$ is the transition operator in level $h$ which takes a state-action pair and returns a distribution over states, 
$r = \{r_h\}_{h = 1}^H$ where
$r_h : \states \times \actions \rightarrow [0, 1]$ is the deterministic reward function\footnote{We assume the reward function is deterministic only for notational convience. Our results can be readily generalized to the case that rewards are stochastic.} in level $h$, 
$H \in \mathbb{Z}_+$ is the planning horizon  (episode length),
and $\mu \in \Delta\left(\states\right)$ is the initial state distribution. 

When the initial distribution $\mu$ and the transition operators $P = \{P_h\}_{h = 1}^H$ are all deterministic, we say $\mdp$ is a {\em deterministic system}.
In this case, we may regard each transition operator $P_h : \states \times \actions \to \states$ as a function that maps state-action pairs to a states.
We note that deterministic systems are special cases of general MDPs.

A policy $\pi$ chooses an action $a \in \actions$ based on the current state $s \in \states$ and the time step $h \in [H]$. 
Formally, $\pi = \{\pi_h\}_{h = 1}^H$ where for each $h \in [H]$, $\pi_h : \states \to \actions$ maps a given state to an action.
The policy $\pi$ induces a trajectory $s_1,a_1,r_1,s_2,a_2,r_2,\ldots,s_{H},a_{H},r_{H}$,
where $s_1 \sim \mu$, $a_1 = \pi_1(s_1)$, $r_1 = r_1(s_1,a_1)$, $s_2 \sim P(s_1,a_1)$, $a_2 = \pi_2(s_2)$, $r_2 = r_2(s_2, a_2)$, etc.

An important concept in RL is the $Q$-function.
For a specific set of reward functions $r = \{r_h\}_{h = 1}^H$, 
given a policy $\pi$, a level $h \in [H]$ and a state-action pair
$(s,a) \in \states \times \actions$, the $Q$-function is defined as
\[
Q_h^\pi(s,a, r) = \expect\left[\sum_{h' = h}^{H} r_{h'}(s_{h'}, a_{h'})\mid s_h =s, a_h = a, \pi\right].
\]
Similarly, the value function of a given state
$s \in \states$ is defined as
\[
V_h^\pi(s, r)=\expect\left[\sum_{h' = h}^{H}r_{h'}(s_{h'}, a_{h'})\mid s_h =s,
  \pi\right].
\]
For a specific set of reward functions $r = \{r_h\}_{h = 1}^H$, 
We use $\pi^*_r$ to denote an optimal policy with respect to $r$, i.e., $\pi^*_r$ is a policy that maximizes \[\expect\left[\sum_{h = 1}^H r_h(s_h, a_h) \mid \pi\right].\]
We also denote $Q_h^*(s,a, r) = Q_h^{\pi^*_r}(s,a, r)$
and $V_h^*(s, r) = V_h^{\pi^*_r}(s, r)$.
We say a policy $\pi$ is $\varepsilon$-optimal with respect to $r$ if \[\expect \left[\sum_{h=1}^{H} r_h(s_h, a_h) \mid \pi\right] \ge \expect \left[\sum_{h=1}^{H} r_h(s_h, a_h)\mid \pi^*_r\right] - \varepsilon.\]
Throughout the paper, when $r$ is clear from the context, we may omit $r$ from $Q_h^\pi(s,a, r)$, $V_h^\pi(s, r)$, $Q_h^*(s,a, r)$, $V_h^*(s, r)$ and $\pi^*_r$.

\subsection{Linear Function Approximation}
\label{sec:lin_mdp}
When applying linear function approximation schemes, it is commonly assumed that the agent is given a feature extractor $\phi : \states \times \actions \to \mathbb{R}^d$ which can either be hand-crafted or a pre-trained neural network that transforms a state-action pair to a $d$-dimensional embedding, and the model or the $Q$-function can be predicted by linear functions of the features. In this section, we consider two different kinds of assumptions: a model-based assumption (linear MDP) and a value-based assumption (linear $Q^*$).

\paragraph{Linear MDP.} The following linear MDP assumption, which was first introduced in~\citep{yang2019sample,jin2019provably}, states that the model of the MDP can be predicted by linear functions of the given features. 
\begin{asmp}[Linear MDP]
	\label{asmp:lin_mdp}
An MDP $\mdp =\left(\states, \actions, P ,r, H, \mu\right)$ is said to be a linear MDP if the followings hold:
\begin{enumerate}
\item there are $d$ unknown signed measures $\mu_h = (\mu_h^{(1)}, \mu_h^{(2)}, \ldots, \mu_h^{(d)})$ such that for any $(s,a,s') \in \states \times \actions \times \states$, $\trans_h\left(s'\mid s,a\right) = \left\langle \mu_h(s'), \phi\left(s,a\right)\right\rangle$;
\item there exists $H$ unknown vectors $\eta_1, \eta_2, \ldots, \eta_H \in \mathbb{R}^d$ such that for any $(s,a) \in \states \times \actions 
$, $r_h(s,a) = \left\langle\phi(s,a),\eta_h\right\rangle$.
\end{enumerate}
As in~\citep{jin2019provably}, we assume for all $(s,a) \in \states \times \actions$ and $h \in [H]$, $\norm{\phi(s,a)} \le 1$, $\norm{\mu_h(S)}_2 \le \sqrt{d}$, and $\norm{\eta}_2 \le \sqrt{d}$.
\end{asmp}
\paragraph{Linear $Q^*$.}
 The following linear $Q^*$ assumption, which is a common assumption in the theoretical RL literature (see e.g.~\citep{du2019provably, Du2020Is}), states that the optimal $Q$-function can be predicted by linear functions of the given features.

\begin{asmp}[Linear $Q^*$]
	\label{asmp:lin_q_star}
An MDP $\mdp =\left(\states, \actions, P ,r, H, \mu\right)$ satisfies the linear $Q^*$ assumption if there exist $H$ unknown vectors $\theta_1, \theta_2, \ldots, \theta_H \in \mathbb{R}^d$ such that for any $(s,a) \in \states \times \actions$,
$Q^*_h(s,a) = \left\langle\phi(s,a),\theta_h\right\rangle$.
We assume $\norm{\phi(s,a)} \le 1$ and $\norm{\theta_h}_2 \le \sqrt{d}$ for all $(s,a) \in \states \times \actions$ and $h \in [H]$.
\end{asmp}

We note that Assumption~\ref{asmp:lin_q_star} is weaker than Assumption~\ref{asmp:lin_mdp}.
Under Assumption~\ref{asmp:lin_mdp}, it can be shown that for any policy $\pi$, $Q_h^{\pi}(\cdot, \cdot)$ is a linear function of the given feature extractor $\phi(\cdot, \cdot)$.
In this paper, we show that Assumption~\ref{asmp:lin_q_star} is {\em strictly weaker} than Assumption~\ref{asmp:lin_mdp} in the reward-free setting, meaning that reward-free RL under Assumption~\ref{asmp:lin_q_star} is {\em exponentially} harder than that under Assumption~\ref{asmp:lin_mdp}.

\subsection{Reward-Free RL}
In the reward-free setting, the goal is to design an algorithm that efficiently explore the state space without the guidance of reward information. 
Formally, there are two phases in the reward-free setting: {\em exploration phase} and {\em planning phase}.

\paragraph{Exploration Phase.} During the exploration phase, the agent interacts with the environment for $K$ episodes.
In the $k$-th episode, the agent chooses a policy $\pi^k$ which induces a trajectory. 
The agent observes the states and actions $s_1^k, a_1^k, s_2^k, a_2^k, \ldots, s_h^k, a_h^k$ as usual, but does not observe any reward values. 
After $K$ episodes, the agent collects a dataset of visited state-actions pairs $\mathcal{D} = \{(s_h^k, a_h^k)\}_{(k,h) \in [K] \times[H]}$ which will be used in the planning phase.

\paragraph{Planning Phase.} During the planning phase, the agent is no longer allowed to interact with the MDP. Instead, the agent is given a set of reward functions $\{r_h\}_{h = 1}^H$ where $r_h : \states \times \actions \rightarrow [0, 1]$ is the deterministic reward function in level $h$,
and the goal here is to output an $\varepsilon$-optimal policy with respect to $r$ using the collected dataset $\mathcal{D}$. 

To measure the performance of an algorithm, we define the {\em sample complexity}  to be the number of episodes $K$ required in the exploration phase to output an $\varepsilon$-optimal policy in the planning phase.

\section{Reward-Free RL for Linear MDPs}
\label{sec:algo}
In this section, we present our reward-free RL algorithm under the linear MDP assumption.
\subsection{The Algorithm}
The exploration phase of the algorithm is presented in Algorithm~\ref{algo:main}, and the planning phase is presented in Algorithm~\ref{algo:batch_ls}.

\begin{algorithm}[!t]
	\caption{Reward-Free RL for Linear MDPs: Exploration Phase}
	\label{algo:main}
	\begin{algorithmic}[1]
		\State \textbf{Input}:  Failure probability $\delta > 0$ and target accuracy $\varepsilon > 0$
		\State $\beta\gets c_{\beta}\cdot dH\sqrt{\log(dH\delta^{-1}\varepsilon^{-1})}$ for some $c_{\beta}>0$
		\State $K \gets c_K \cdot d^3H^6\log(dH  \delta^{-1}\varepsilon^{-1}) / \varepsilon^2$ for some $c_K > 0$
		\For{$k=1,2,\ldots K$}
		\State $Q^k_{H+1}(\cdot,\cdot)\gets 0$ and $V^k_{H + 1}(\cdot) = 0$
		\For{$h=H, H-1, \ldots, 1$}
		\State $\Lambda_h^k\gets \sum_{\tau=1}^{k-1}\phi(s^{\tau}_h, a_h^{\tau}) \phi(s^{\tau}_h, a_h^{\tau})^\top + I$ \label{line:covariance_exploration}
		\State $u^k_h(\cdot, \cdot)\gets\min\left\{\beta\cdot \sqrt{\phi(\cdot,\cdot)^\top(\Lambda_h^k)^{-1}\phi(\cdot,\cdot)}, H\right\}$ \label{line:ucb}
		\State Define the exploration-driven reward function 
		$
		r_{h}^k(\cdot, \cdot) \gets u^k_h(\cdot, \cdot)/H
		$
		\State $w_h^k\gets (\Lambda_h^k)^{-1}
		\sum_{\tau=1}^{k-1}\phi(s_h^\tau, a_h^\tau)\cdot V^{k}_{h+1}(s_{h+1}^{\tau})
		$
		\State $Q_h^k(\cdot, \cdot)\gets \min\{(w_h^k)^\top \phi(\cdot, \cdot) +r_{h}^k(\cdot, \cdot) +  u_h^{k}(\cdot, \cdot), H\}$ and $V_h^k(\cdot) = \max_{a \in \actions} Q_h^k(\cdot, a)$
		\State $\pi^k_h(\cdot) \gets \argmax_{a \in \actions} Q_h^k(\cdot, a)$
		\EndFor
		\State Receive initial state $s_1^k\sim \mu$
		\For{$h=1, 2, \ldots H$}
			\State Take action $a_h^k\gets \pi^k(s_h^k)$ and observe $s_{h+1}^k \sim P_h(s_h^k, a_h^k)$
		\EndFor
		\EndFor
		\State \Return $\mathcal{D}\gets\{(s^k_h, a^k_h)\}_{(k,h)\in [K]\times[H]}$
	\end{algorithmic}
	\label{algo:explore}
\end{algorithm}

\begin{algorithm}[!t]
	\caption{Reward-Free RL for Linear MDPs: Planning Phase}
\label{algo:batch_ls}
	\begin{algorithmic}[1]
		\State \textbf{Input}: Dataset $\mathcal{D}=\{(s^k_h, a^k_h)\}_{(k,h)\in [K]\times[H]}$, reward functions $r = \{r_h\}_{h \in [H]}$
		\State $Q_{H+1}(\cdot,\cdot)\gets 0$ and $V_{H + 1}(\cdot) = 0$
		\For{step $h=H, H-1, \ldots, 1$}
		\State $\Lambda_h\gets \sum_{\tau=1}^{K}\phi(s^{\tau}_h, a_h^{\tau})\phi(s^{\tau}_h, a_h^{\tau})^\top + I$ \label{line:covariance_planning}
		\State Let $u_h(\cdot, \cdot)\gets\min\left\{\beta\cdot \sqrt{\phi(\cdot,\cdot)^\top(\Lambda_h)^{-1}\phi(\cdot,\cdot)}, H\right\}$ \label{line:ucb_planning}
\State $w_h\gets (\Lambda_h)^{-1}
		\sum_{\tau=1}^{K}\phi(s_h^\tau, a_h^\tau)\cdot 
		V_{h+1}(s_{h+1}^{\tau}, a)
		$
		\State $Q_h(\cdot, \cdot)\gets \min\{(w_h)^\top \phi(\cdot, \cdot) + r_{h}(\cdot, \cdot) +  u_h(\cdot, \cdot), H\}$ and $V_h(\cdot) = \max_{a \in \actions}Q_h(\cdot, a)$
		\State $\pi_h(\cdot) \gets \argmax_{a \in \actions} Q_h(\cdot, a)$
		\EndFor

		\State \textbf{Return} $\pi = \{\pi_h\}_{h \in [H]}$
	\end{algorithmic}
\end{algorithm}

\paragraph{Exploration Phase.} During the exploration phase of the algorithm, we employ the least-square value iteration (LSVI) framework introduced in~\citep{jin2019provably}.
In each episode, we first update the parameters $(\Lambda_h, w_h)$ that are used to calculate the $Q$-functions, and then execute the greedy policy with respect to the updated $Q$-function to collect samples. 
As in~\citep{jin2019provably}, to encourage exploration, Algorithm~\ref{algo:main} adds an upper-confidence bound (UCB) bonus function $u_h$.

The main difference between Algorithm~\ref{algo:main} and the one in~\citep{jin2019provably} is the definition of the \emph{exploration-driven} reward function.
Since the algorithm in~\citep{jin2019provably} is designed for the standard RL setting, the agent can obtain reward values by simply interacting with the environment.
On the other hand, in the exploration phase of the reward-free setting, the agent does not have any knowledge about the reward function. 
In our algorithm, in each episode, we design an exploration-driven reward function which is defined to be $r_h(\cdot, \cdot) = u_h(\cdot, \cdot) / H$, where $u_h(\cdot, \cdot)$ is the UCB bonus function defined in Line~\ref{line:ucb}.
Note that we divide $ u_h(\cdot, \cdot)$ by $H$ so that $r_h(\cdot, \cdot)$ always lies in $[0, 1]$.
Intuitively, such a reward function encourages the agent to explore state-action pairs where the amount of uncertainty (quantified by $u_h(\cdot, \cdot)$) is large.
After sufficient number of episodes, the uncertainty of all state-action pairs should be low on average, since otherwise the agent would have visited those state-action pairs with large uncertainty as guided by the reward function.

\paragraph{Planning Phase.}
After the exploration phase, the returned dataset contains sufficient amount of information for the planning phase.
In the planning phase (Algorithm~\ref{algo:batch_ls}), for each step $h = H, H - 1, \ldots, 1$, we optimize a least squares predictor to predict the $Q$-function, and return the greedy policy with respect to the predicted $Q$-function. 
During the planning phase, we still add an UCB bonus function $u_h(\cdot, \cdot)$ to guarantee optimism. 
However, as mentioned above and will be made clear in the analysis, since the agent has acquired sufficient information during the exploration phase, $u_h(\cdot, \cdot)$ should be small on average, which implies the returned policy is near-optimal. 
\subsection{Analysis}
In this section we outline the analysis of our algorithm.
The formal proof is deferred to the supplementary material. 
We first give the formal theoretical guarantee of our algorithm. 
\begin{thm}\label{thm:main}
After collecting $O\left(d^3H^6\log(dH \delta^{-1}\varepsilon^{-1}) / \varepsilon^2\right)$ trajectories during the exploration phase,
with probability $1 - \delta$, our algorithm outputs an $\varepsilon$-optimal policy for an arbitrary number of reward functions satisfying Assumption~\ref{asmp:lin_mdp} during the planning phase.
\end{thm}

Now we show how to prove Theorem~\ref{thm:main}. Our first lemma shows that the estimated value functions $V^{k}$ are optimistic with high probability, and the summation of $V^k_1(s_1^k)$ should be small.
\begin{lem}\label{lem:sum_V}
With probability $1-\delta / 2$, for all $k \in [K]$, 
\[
V^{*}_1(s_1^k, r^k)
\le
V^{k}_1(s_1^k)
\]
and 
\[
\sum_{k=1}^{K}V^{k}_1(s_1^k)
\le c\sqrt{d^3H^4K\cdot\log(dKH/\delta)}
\]
for some constant $c > 0$ where $V^k_1(\cdot)$ is as defined in Algorithm~\ref{algo:main}.
\end{lem}
Note that the definition of the exploration driven reward function $r^k$ used in the $k$-th episode depends only on samples collected during the first $k - 1$ episodes. 
Therefore, the first part of the proof is nearly identical to that of Theorem 3.1 in~\citep{jin2019provably}.
To prove the second part of the lemma, we first recursively decompose $V^{k}_1(s_1^k)$ (similar to the standard regret decomposition for optimistic algorithms), and then use the fact that $r_h(\cdot) = u_h(\cdot) / H$ and the elliptical potential lemma in~\citep{abbasi2012online} to given an upper bound on $\sum_{k = 1}^K V^{k}_1(s_1^k)$.
The formal proof is provided in the supplementary material.

Our second lemma shows that with high probability, if one divides the bonus function $u_h(\cdot, \cdot)$ (defined in Line~\ref{line:ucb_planning} in Algorithm~\ref{algo:batch_ls}) by $H$ and uses it as a reward function, then the optimal policy has small cumulative reward on average. 
\begin{lem}\label{lem:small_V}
	With probability $1-\delta / 4$, for the function $u_h(\cdot, \cdot)$ defined in Line~\ref{line:ucb_planning} in Algorithm~\ref{algo:batch_ls}, we have
\[
\EE_{s\sim\mu}\left[V^{*}_1(s, u_h / H)\right]
\le c'\sqrt{d^3 H^4\cdot \log(dKH/\delta)/K}
\]
	for some absolute constant $c'>0$.
\end{lem}
To prove Lemma~\ref{lem:small_V}, we first note that $\EE_{s\sim\mu}\left[\sum_{k=1}^{K}V^{*}_1(s, r^k)\right]$ is close to $\sum_{k=1}^{K}V^{*}_1(s_1^k, r^k)$ by Azuma–Hoeffding inequality
and  $\sum_{k=1}^{K}V^{*}_1(s_1^k, r^k)$ can be bounded by using Lemma~\ref{lem:sum_V}.
Moreover, for $\Lambda_h$ defined in Line~\ref{line:covariance_planning} in Algorithm~\ref{algo:batch_ls}, we have $\Lambda_h \succeq \Lambda_h^k$ for all $k \in [K]$ where $\Lambda_h^k$ is defined in Line~\ref{line:covariance_exploration} in Algorithm~\ref{algo:main}, 
which implies $u_h(\cdot, \cdot) / H \le r_h^k(\cdot, \cdot)$ for all $k \in [K]$.
Therefore, we have \[\EE_{s\sim\mu}\left[V^{*}_1(s, u_h / H)\right] \le \EE_{s\sim\mu}\left[V^{*}_1(s, r^k)\right]\] for all $k \in [K]$, which implies the desired result.

Our third lemma states the estimated $Q$-function is always optimistic, and is upper bounded by $r_h(\cdot, \cdot) 
+ \sum_{s'}P_{h}(s' \mid \cdot, \cdot)V_{h+1}(s')$ plus the UCB bonus function $u_h(\cdot, \cdot)$. 
The lemma can be proved using the same concentration argument as in~\citep{jin2019provably}.
\begin{lem}\label{lem:confidence_planning}
With probability $1-\delta / 2$, 
for an arbitrary number of reward functions satisfying Assumption~\ref{asmp:lin_mdp} and all $h \in [H]$,
we have
\[
Q^*_h(\cdot,\cdot, r)\le Q_h(\cdot, \cdot) 
\le
r_h(\cdot, \cdot) 
+ \sum_{s'}P_{h}(s' \mid \cdot, \cdot)V_{h+1}(s')
+
2u_h(\cdot, \cdot)
.\]
\end{lem}

Now we sketch how to prove Theorem~\ref{thm:main} by combining Lemma~\ref{lem:small_V} and Lemma~\ref{lem:confidence_planning}.
Note that With probability $1 - \delta$, the events defined in Lemma~\ref{lem:small_V} and Lemma~\ref{lem:confidence_planning} both hold.
Conditioning on both events, we have
\begin{align*}
&\EE_{s\sim \mu}[V^{*}_1(s, r)
- 
V^{\pi}_1(s, r)]
\le 
\EE_{s\sim \mu}[V_1(s)
- 
V^{\pi}_1(s, r)]\\
\le&
\EE_{s\sim \mu}[V^{\pi}_1(s, u)]
\le 
\EE_{s\sim \mu}[V^{*}_1(s, u)]
\le c'H\sqrt{d^3H^4\cdot \log(dKH/\delta)/K},
\end{align*}
where the first inequality follows by Lemma~\ref{lem:confidence_planning}, the second inequality follows by Lemma~\ref{lem:confidence_planning} and decomposing the $V$-function recursively, the third inequality follows by the definition of $V^*$, and the last inequality follows by Lemma~\ref{lem:small_V}.

 \section{Lower Bound for Reward-Free RL under Linear $Q^*$ Assumption}\label{sec:hardness}
In this section we prove lower bound for reward-free RL under the linear $Q^*$ assumption.
We show that there exists a class of MDPs which satisfies Assumption~\ref{asmp:lin_q_star}, 
such that any reward-free RL algorithm requires exponential number of samples during the exploration phase in order to find a near-optimal policy during the planning phase.
In particular, we prove the following theorem.
\begin{thm}\label{thm:lb}
	There exists a class of deterministic systems that satisfy Assumption~\ref{asmp:lin_q_star} with $d = \mathrm{poly}(H)$, such that any reward-free algorithm requires at least $\Omega(2^H)$ samples during the exploration phase in order to find a $0.1$-optimal policy with probability at least $0.9$ during the planning phase for a given set of reward functions $r = \{r_h\}_{h = 1}^H$.
\end{thm}

Since deterministic systems are special cases of general MDPs, the hardness result in Theorem~\ref{thm:lb} applies to general MDPs as well.
In the remaining part of this section, we describe the construction of the hard instance and outline the proof of Theorem~\ref{thm:lb}.
\begin{figure}[!h]
\centering
\includegraphics[scale=0.45]{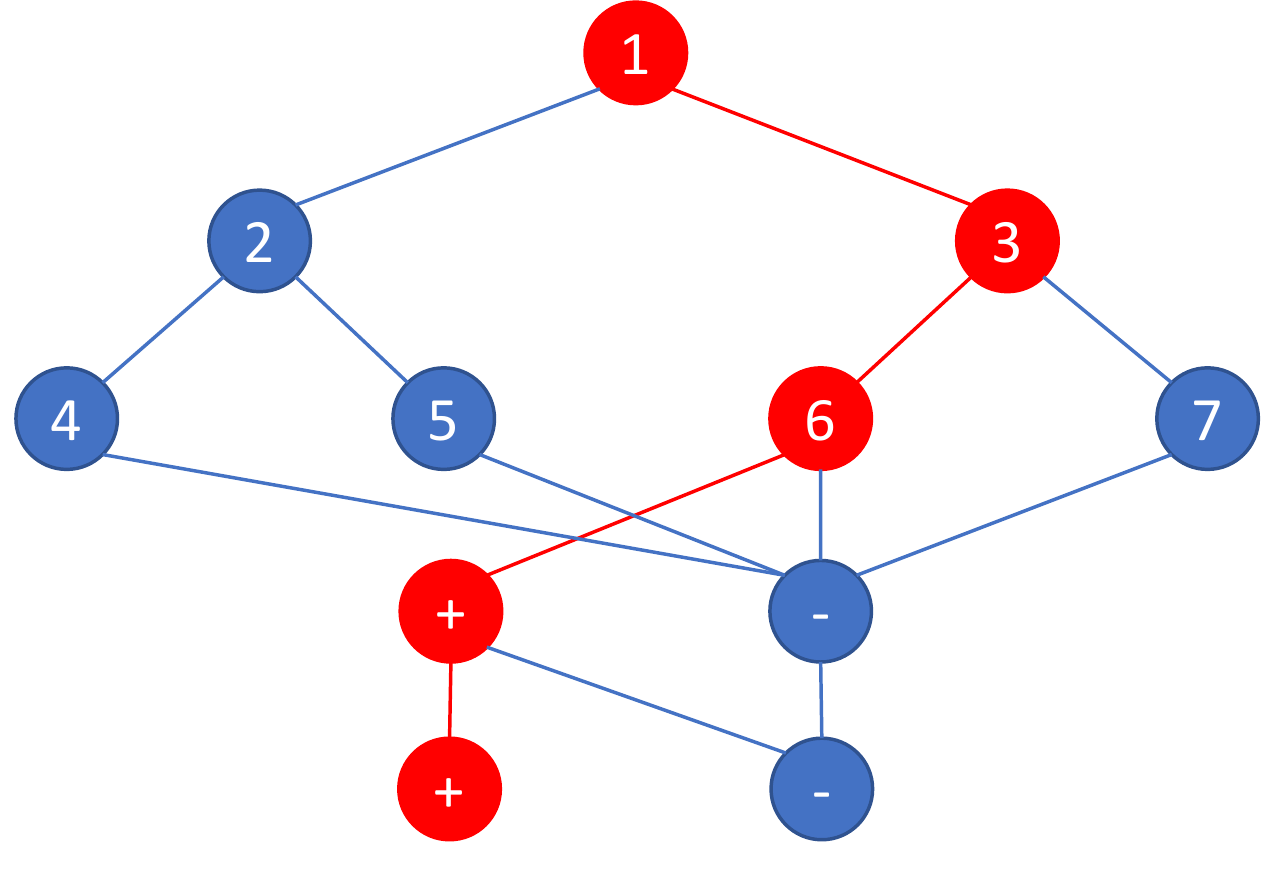}
\caption{An illustration of the hard instance with $H = 5$. Red states and transitions are those on the optimal trajectory $s_1^*, a_1^*, s_2^*, a_2^*, \ldots, s_{H - 1}^*, a_{H - 1}^*, s_H^*, a_H^*$.}
\end{figure}

\paragraph{State Space and Action Space.}
In the hard instance, there are $H$ levels of states \[\cS = \cS_1\cup \cS_2 \cup\ldots \cup \cS_H\] where $\cS_h$ contains all states that can be reached in level $h$.
The action space $\cA=\{0, 1\}$.
For each $h\in [H-2]$, we represent each state in $\cS_h$ by an integer in $[2^{h-1}, 2^h)$, i.e., $\cS_1 = \{1\}$, $\cS_{2}=\{2, 3\}$, $\cS_3=\{4, 5, 6, 7\}$, etc.
We also have $S_{H - 1} = \{s_{H - 1}^+, s_{H - 1}^-\}$ and  $S_H=\{s_H^{+}, s_H^{-}\}$.
The initial states is $1 \in \cS_1$.

\paragraph{Transition.} For each $h \in[H-3]$, for each $s \in \cS_h$, $P_h(s, a)$ is fixed and thus known to the algorithm. 
In particular, for each $h \in [H - 3]$, for each $s \in \cS_h$, 
we define $P_h(s, a) = 2s + a \in \cS_{h + 1}$ where $a \in \{0, 1\}$.
We will define the transition operator for those states $s \in \cS_{H - 2} \cup \cS_{H - 1}$ shortly.

\paragraph{Feature Extractor.}
For each $h \in [H - 2]$, for each $(s, a) \in \cS_h \times \actions$, we define $\phi(s, a) \in \mathbb{R}^{d}$ so that $\|\phi(s, a)\|_2 = 1$ and for any $(s', a') \in \cS_h \times \actions \setminus \{(s, a)\}$, we have $|\left(\phi(s, a)\right)^{\top} \phi(s', a')| \le 0.01$. 
In the supplementary material, we use the Johnson–Lindenstrauss Lemma~\citep{johnson1984extensions} to show that such feature extractor exists if $d = \mathrm{poly}(H)$.
We note that similar hard instance constructions for the feature extractor have previously appeared in~\citep{Du2020Is}.
However, we stress that our construction is different from that in~\citep{Du2020Is}.
In particular, in our hard instance the optimal $Q$-function is exactly linear, while for the hard instance in~\citep{Du2020Is}, the optimal $Q$-function is only approximately linear. 
Moreover, we focus on the reward-free setting while~\citet{Du2020Is} focused on the standard RL setting.

For all states $s \in \cS_{H - 1}$, we define
\[
\phi(s, a) = \begin{cases}
[1, 0, 0, \ldots, 0]^{\top} & s = s_{H - 1}^+, a = 0\\
[0, 1, 0, \ldots, 0]^{\top} & s = s_{H - 1}^+, a = 1\\
[0, 0, 0, \ldots, 0]^{\top} & s  = s_{H - 1}^-\\
\end{cases}.
\]
Finally, for all states $s \in \cS_H$, we define 
\[
\phi(s, a) = \begin{cases}
[1, 0, 0, \ldots, 0]^{\top} & s = s_H^+, a = 0\\
[0, 0, 0, \ldots, 0]^{\top} & \text{otherwise}
\end{cases}.
\]

\paragraph{The Hard MDPs.}
By Yao's minimax principle~\citep{yao1977probabilistic}, to prove a lower bound for randomized algorithms, it suffices to define a hard distribution and show that any deterministic algorithm fails for the hard distribution. 
We now define the hard distribution.
We first define the transition operator $P_{H - 2}(s, a)$ for those states $s \in \cS_{H - 2}$.
To do this, we first pick a state-action pair $(s_{H - 2}^*, a_{H - 2}^*)$ from $\states_{H - 2} \times \actions$ uniformly at random, and define
\[
P_{H - 2}(s, a) = \begin{cases}
s_{H - 1}^+  & s = s_{H - 2}^*, a = a_{H - 2}^* \\
s_{H - 1}^- & \text{otherwise}
\end{cases}.
\]
To define the transition function $P_{H - 1}(s, a)$ for those states $s \in \cS_{H - 1}$,
we pick a random action $a^*_{H - 1}$ from $\{0, 1\}$ uniformly at random, and define
\[
P_{H - 1}(s, a) = \begin{cases}
s_H^+  & s = s_{H - 1}^+, a = a^*_{H - 1} \\
s_H^- & \text{otherwise}
\end{cases}.
\]

\paragraph{The Reward Function.}
We now define the optimal $Q$-function which automatically implies a set of reward function $r = \{r_h\}_{h = 1}^H$.
During the planning phase, the agent will receive $r$ as the reward functions.
By construction, there exists a unique trajectory $s_1^*, a_1^*, s_2^*, a_2^*, \ldots, s_{H - 1}^*, a_{H - 1}^*, s_H^*, a_H^*$ with $(s_H^*, a_H^*) = (s_H^+, 0)$.
For each $h \in [H ]$, we define $\theta_h$ in Assumption~\ref{asmp:lin_q_star} as $\phi(s_h^*, a_h^*) / 2$.
This implies that for each $(s, a) \in \cS_{H} \times \actions$,
\[
r_H(s, a) = Q^*_{H}(s, a) = \begin{cases}
0.5 & s = s^*_{H}, a = a^*_{H} \\
0 & \text{otherwise}
\end{cases}.
\]
For each $(s, a) \in \cS_{H - 1} \times \actions$, we have
\[
Q^*_{H - 1}(s, a) = \begin{cases}
0.5 & s = s^*_{H - 1}, a = a^*_{H - 1} \\
0 & \text{otherwise}
\end{cases},
\]
which implies that $r_{H - 1}(s, a) = 0$ for all $(s, a) \in \states_{H - 1} \times \actions$.
Now for each $h \in [H - 2]$, for each $(s, a) \in \cS_{h} \times \actions$, we define
$
r_h(s_h, a_h) = Q^*_h(s_h, a_h) - \max_{a \in \actions} Q^*_{h + 1}(\cdot, a)
$
so that the Bellman equations hold. 
Moreover, by construction, for each $h \in [H ]$, we have $Q^*_h(s, a) = 0.5$ when $(s, a) = (s_h^*, a_h^*)$, and $|Q^*_h(s, a)| \le 0.01$ when $(s, a) \neq (s_h^*, a_h^*)$ and thus $r_h(\cdot, \cdot) \in [-0.02, 0.5]$.\footnote{Note that this is slightly different from the assumption that $r_h(\cdot, \cdot) \in [0, 1]$. However, this can be readily fixed by shifting all reward values by $0.02$.}

\paragraph{Proof of Hardness.}
Now we sketch the final proof of the hardness result.
We define $\mathcal{E}$ to be the event that for all $(s, a) \in \mathcal{D}$ where $\mathcal{D}$ are the state-action pairs collected by the algorithm, we have $s \neq s_{H - 1}^* = s_{H - 1}^+$.
For any deterministic algorithm, we claim that if the algorithm samples at most $2^H / 100$ trajectories during the exploration phase, with probability at least $0.9$ over the randomness of the distribution of MDPs, $\mathcal{E}$ holds. 
This is because the feature extractor is fixed and thus the algorithm receives the same feedback before reaching $s_{H - 1}^+$.
Since there are $2^{H - 2}$ state-action pairs $(s, a) \in \cS_{H - 2} \times \actions$ and only one of them satisfies $P_{H - 2}(s, a) = s_{H - 1}^+$, and the algorithm samples at most $2^H / 100$ trajectories during the exploration phase, $\mathcal{E}$ holds with probability at least $0.9$.

Now during the planning phase, by construction of the optimal $Q$-function, the only $0.1$-optimal policy is $\pi_h(s_{h}^*) = a_h^*$.
However, conditioned on $\mathcal{E}$, any deterministic algorithm correctly output $\pi_{H - 1}(s_{H - 1}^*) = a_{H - 1}^*$ with probability at most $0.5$, since conditioned on $\mathcal{E}$, $\mathcal{D}$ does not contain $s_{H - 1}^*$, and the set of reward functions $r = \{r_h\}_{h = 1}^H$ also does not depend on $a_{H - 1}^*$.
Therefore, during the planning phase of the algorithm, a $0.1$-optimal policy is found with probability at most $0.6 < 0.9$.

\section{Conclusion}
\label{sec:con}
This paper provides both positive and negative results for reward-free RL with linear function approximation.
Our results imply three new exponential separations: 1) linear MDP v.s. linear $Q^*$,  2) standard RL v.s. reward-free RL, and 3) query with a simulator v.s. query without a simulator.
An interesting future direction is to generalize our results to more general function classes using techniques, e.g., in~\citep{wen2013efficient,ayoub2020model,wang2020provably}.

\section*{Acknowledgments}
RW and RS are supported in part by NSF IIS1763562,
AFRL CogDeCON FA875018C0014, and DARPA SAGAMORE HR00111990016.
SSD is supported by NSF grant DMS-1638352 and the Infosys Membership.

\bibliography{simonduref}
\bibliographystyle{plainnat}

\newpage
\appendix
\section{Missing Proofs in Section~\ref{sec:algo}}
In this section, for all $(k,h)\in [K]\times[H]$, we denote 
\[
\phi_h^{k} := \phi(s^{k}_h, a^{k}_h).
\]
In Algorithm~\ref{algo:explore} and \ref{algo:batch_ls}, we recall that
\[
\beta= c_{\beta}dH\sqrt{\log(dH/\delta/\eps)}.
\]
Since
$K= c_K \cdot d^3H^6\log(dH  \delta^{-1}\varepsilon^{-1})/\varepsilon^2$, 
we have 
\[
\beta\ge 
c_{\beta}dH\sqrt{\log(dHK/\delta)}
\]
for appropriate choices of $c_{\beta}$ and $c_{K}$.
\subsection{Proof of Lemma~\ref{lem:sum_V}}
To prove Lemma~\ref{lem:sum_V}, we need a concentration lemma similar to Lemma~B.3 in~\citep{jin2019provably}.
\begin{lem}
\label{lem:concentration}
Suppose Assumption~\ref{asmp:lin_mdp} holds.  Let $\cE$ be the event that for all $ (k,h)\in [K]\times[H]$,
\[ 
\left\|\sum_{\tau=1}^{k-1}\phi_h^{\tau}
\left(V_{h+1}^{k}(s^{\tau}_{h+1})-\sum_{s'\in \cS}P_h(s'|s_h^\tau, a_h^\tau) V_{h+1}^{k}(s')\right)
\right\|_{(\Lambda_h^k)^{-1}}\le c\cdot dH \sqrt{\log(dKH/\delta)}
\]
for some absolute constant $c>0$.
Then  $\Pr[\gE]\ge 1-\delta/4$.
\end{lem}
\begin{proof}
The proof is nearly identical to that of Lemma~B.3 in~\citep{jin2019provably}.
The only deference in our case is that we have a different reward functions at different episodes. 
However, 
note that in our case
\[
r^{k}_h(\cdot, \cdot)
= u_h^k(\cdot, \cdot)/H
\]
and hence
\[
r^{k}_h(\cdot, \cdot) + u_h^k(\cdot, \cdot)
 =  (1+1/H)\cdot\min\left\{\beta\cdot \sqrt{\phi(\cdot,\cdot)^\top(\Lambda_h^k)^{-1}\phi(\cdot,\cdot)}, H\right\}.
\]
Thus our value function $V^{k}_{h+1}$ is of the form 
\[
V(\cdot)
:=\min\left\{\max_{a}w^{\top }\phi(\cdot, a)+ \beta\cdot (1+1/H)\cdot \sqrt{\phi(\cdot, a)^\top \Lambda^{-1} \phi(\cdot, a)}, H\right\}
\]
for some $\Lambda\in \RR^{d\times d}$, and $w \in \mathbb{R}^d$.
Therefore, the value function shares exactly the same function class as that in Lemma~D.6 in~\citep{jin2019provably}. 
The rest of the proof follow similarly.
\end{proof}
We are now ready to prove  Lemma~\ref{lem:sum_V}.
\begin{proof}[Proof of Lemma~\ref{lem:sum_V}]
In our proof, we condition on the event $\cE$ defined in Lemma~\ref{lem:concentration}, which holds with probability at least $1-\delta/4$.
Since $P_h(s'|s,a) = \phi(s,a)^\top \mu_h(s')$,
we have
\[
\sum_{s'\in \cS }P_h(s'|s,a) V_{h+1}^k(s') = 
\phi(s,a)^\top \wt{w}^{k}_{h}
\]
where
\[
\wt{w}^{k}_h := \sum_{s'\in \cS }\mu_h(s') V_{h+1}^k(s')
\]
is an unknown vector.
By Assumption~\ref{asmp:lin_mdp}, 
$ \sum_{s'\in \cS }\mu_h(s') \le \sqrt{d}$.
Therefore,
\[
\|\wt{w}^{k}_h\|_{2}
\le H\sqrt{d}.
\]
We thus have, for all $(h, k)\in [H] \times [K]$ and $(s, a) \in \states \times \actions$, 
\begin{align*}
&\phi(s,a)^\top w^{k}_h - \sum_{s' \in \states} P_h(s' \mid s,a)^\top V_{h+1}^k(s')\\
=&
\phi(s,a)^\top (\Lambda^{k}_h)^{-1}
\sum_{\tau=1}^{k-1}\phi_h^\tau\cdot
V_{h+1}^k(s_{h+1}^{\tau}) - \sum_{s'\in \cS}P_h(s'|s,a) V_{h+1}^k(s')\\
=&
\phi(s,a)^\top (\Lambda^{k}_h)^{-1}\left(
\sum_{\tau=1}^{k-1}\phi_h^\tau
V_{h+1}^k(s_{h+1}^{\tau}) 
- \Lambda^{k}_h \wt{w}^{k}_{h}\right)\\
=&
\phi(s,a)^\top (\Lambda^{k}_h)^{-1}\left(
\sum_{\tau=1}^{k-1}\phi_h^\tau
V_{h+1}^k(s_{h+1}^{\tau}) 
-  \wt{w}^{k}_{h} 
- 
\sum_{\tau=1}^{k-1}\phi_h^\tau(\phi_h^\tau)^\top
\wt{w}^{k}_{h} 
\right)\\
=&
\phi(s,a)^\top (\Lambda^{k}_h)^{-1}\left(
\sum_{\tau=1}^{k-1}\phi_h^\tau
\left(V_{h+1}^k(s_{h+1}^{\tau}) - \sum_{s'}P_h(s'|s_h^\tau, a_h^\tau)V_{h+1}^{k}(s')\right)
-  \wt{w}^{k}_{h}
\right).
\end{align*}
We have,
\begin{align*}
&\Bigg|\phi(s,a)^\top (\Lambda^{k}_h)^{-1}\left(
\sum_{\tau=1}^{k-1}\phi_h^\tau
\left(V_{h+1}^k(s_{h+1}^{\tau}) - \sum_{s'}P_h(s'|s_h^\tau, a_h^\tau)V_{h+1}^{k}(s')\right)\right)\Bigg|\\
=& 
\left|\phi(s,a)^\top (\Lambda^{k}_h)^{-1/2}(\Lambda^{k}_h)^{-1/2}\left(
\sum_{\tau=1}^{k-1}\phi_h^\tau
\left(V_{h+1}^k(s_{h+1}^{\tau}) - \sum_{s'}P_h(s'|s_h^\tau, a_h^\tau)V_{h+1}^{k}(s')\right)\right)\right|\\
\le&
\|\phi(s,a)\|_{(\Lambda^{k}_h)^{-1}}\cdot
\left\|\sum_{\tau=1}^{k-1}\phi_h^\tau
\left(V_{h+1}^k(s_{h+1}^{\tau}) - \sum_{s'}P_h(s'|s_h^\tau, a_h^\tau)V_{h+1}^{k}(s')\right)\right\|_{(\Lambda^{k}_h)^{-1}}.
\end{align*}
By Lemma~\ref{lem:concentration},  we have
\begin{align*}
&\left|\phi(s,a)^\top (\Lambda^{k}_h)^{-1}\left(
\sum_{\tau=1}^{k-1}\phi_h^\tau
\left(V_{h+1}^k(s_{h+1}^{\tau}) - \sum_{s'}P_h(s'|s_h^\tau, a_h^\tau)V_{h+1}^{k}(s')\right)\right)\right|\\
\le& 
cdH\sqrt{\log(dKH/\delta)}\cdot 
\|\phi(s,a)\|_{(\Lambda^{k}_h)^{-1}}.
\end{align*}
Moreover, we have
\[
\left|\phi(s,a)^\top (\Lambda^{k}_h)^{-1}\wt{w}^{k}_{h}
\right|
\le \|\phi(s,a)\|_{(\Lambda^{k}_h)^{-1}}\cdot 
\|\wt{w}^{k}_{h}\|_{(\Lambda^{k}_h)^{-1}}
\le \|\phi(s,a)\|_{(\Lambda^{k}_h)^{-1}}\cdot H\sqrt{d}.
\]
Therefore, we have
\begin{align*}
&\left|
\phi(s,a)^\top w^{k}_h - \sum_{s' \in \states} P_h(s' \mid s,a) V_{h+1}^k(s')
\right|\\
\le& cdH\sqrt{\log(dKH/\delta)}\cdot 
\|\phi(s,a)\|_{(\Lambda^{k}_h)^{-1}}
+ \|\phi(s,a)\|_{(\Lambda^{k}_h)^{-1}}\cdot H\sqrt{d}\\
\le&
c_{\beta}dH\sqrt{\log(dKH/\delta)}\cdot 
\|\phi(s,a)\|_{(\Lambda^{k}_h)^{-1}}\\
=&\beta  \cdot \|\phi(s,a)\|_{(\Lambda^{k}_h)^{-1}}.
\end{align*}

Now we prove the first part of the lemma.

\paragraph{First Part.} 
Our proof is by induction on $h$.
Indeed, for  $h=H+1$, it holds that for all $s \in \states$,
\[
V_{H+1}^{*}(s, r^k)\le V_{H+1}^k(s)
\]
since $V_{H+1}^* = V_{H+1}^k = 0$.
Suppose for some $h\in [H]$, it holds that for all $s \in \states$, 
\[
V_{h+1}^{*}(s, r^k)\le V_{h+1}^k(s).
\]
Then we have
\begin{align*}
&V_{h}^{*}(s, r^k)
= \max_{a\in \cA}
\left(r^k_h(s,a) + \sum_{s' \in \states} P_{h}(s' \mid s,a) V_{h+1}^{*}(\cdot, r^k)\right)\\
\le& 
\max_{a\in \cA}
\left(r^k_h(s,a) + \sum_{s' \in \states}P_{h}(s' \mid s,a) V_{h+1}^{k}(s', r^k)\right).
\end{align*}
Notice that for all $(s, a) \in \states \times \actions$,
\[
\sum_{s' \in \states} P_{h}(s' \mid s,a)^\top V_{h+1}^{k}(s', r^k)
\le \phi(s,a)^\top w^{k}_h 
+ \beta\cdot \|\phi(s,a)\|_{(\Lambda^{k}_h)^{-1}}.
\]
We have
\[
V_{h}^{*}(s, r^k)
\le 
\min\left\{\max_{a\in \cA}
\left(r^k_h(s,a) + \phi(s,a)^\top w^{k}_h 
+ \beta\cdot \|\phi(s,a)\|_{(\Lambda^{k}_h)^{-1}}\right),
H\right\}
 = V_{h}^k(s)
\]
as desired.

\paragraph{Second Part.}
To prove the second part, for all  $(k,h)\in [K]\times[H - 1]$, we denote 
\[
\xi_h^k = \sum_{s'\in \cS}P(s'|s_h^k,a_h^k) V_{h+1}^k(s')
  - V_{h+1}^k(s_{h+1}^k).
\] 
Conditioned on $\cE$, 
\begin{align*}
\sum_{k=1}^KV_1^k(s_1^k)
&\le \sum_{k=1}^K
\left(r^k_1(s_1^k,a_1^k) + \phi(s_1^k,a_1^k) ^\top w^{k}_h 
+ \beta\cdot \|\phi(s_1^k,a_1^k) \|_{(\Lambda^{k}_1)^{-1}}\right)\\
&= 
\sum_{k=1}^K
\left(\phi(s_1^k,a_1^k) ^\top w^{k}_h 
+ (1+1/H)\cdot\beta\cdot \|\phi(s_1^k,a_1^k) \|_{(\Lambda^{k}_1)^{-1}}\right)\\
&\le \sum_{k=1}^K
\left( \sum_{s'\in \cS}P(s'|s_1^k,a_1^k) V_{2}^k(s')
+ (2+1/H)\cdot\beta\cdot \|\phi(s_1^k,a_1^k) \|_{(\Lambda^{k}_1)^{-1}}\right)\\
&\le 
\sum_{k=1}^K
\left(\xi_1^k+ V_{2}^k(s_2^k)
+ (2+1/H)\cdot\beta\cdot \|\phi(s_1^k,a_1^k) \|_{(\Lambda^{k}_1)^{-1}}\right)\\
&\le \ldots\\
&\le 
\sum_{k=1}^K\sum_{h=1}^{H - 1}\xi_h^k + 
\sum_{k=1}^K\sum_{h=1}^H(2+1/H)\cdot \beta\cdot \|\phi(s_h^k,a_h^k) \|_{(\Lambda^{k}_h)^{-1}}.
\end{align*}
Note that for each $h\in[H - 1]$, $\{\xi_{h}^k\}_{k=1}^K$ is a martingale difference sequence with $|\xi_{h}^k|\le H$.
Define $\cE'$ to be the even that  
\[
\left|\sum_{k=1}^K\sum_{h=1}^{H - 1}\xi_h^k\right|
\le c'H^2\sqrt{K\log(KH/\delta)}.
\]
By Azuma–Hoeffding inequality, we have $\Pr[\cE']\ge 1 - \delta/4$.

Next, we have,
\begin{align*}
\sum_{k=1}^K\sum_{h=1}^H \|\phi(s_h^k,a_h^k) \|_{(\Lambda^{k}_h)^{-1}}
\le  \sqrt{KH \sum_{k=1}^K\sum_{h=1}^H \phi(s_h^k,a_h^k)^\top (\Lambda^{k}_h)^{-1} \phi(s_h^k,a_h^k)}.
\end{align*}
By Lemma~D.2 in~\citep{jin2019provably}, we have
\[
\sum_{h=1}^H\sum_{k=1}^K\phi(s_h^k,a_h^k)^\top (\Lambda^{k}_h)^{-1} \phi(s_h^k,a_h^k)
\le 2dH\log(K).
\]
Conditioned on $\cE\cap \cE'$ which holds with probability at least $1-\delta/2$, we have
\begin{align*}
\sum_{k=1}^KV_1^k(s_1^k)
&\le 
c'H^2\sqrt{K\log(KH/\delta)} 
 + (2+1/H)\cdot \beta
 \cdot \sqrt{KH \cdot 2dH\log(K)}\\
&\le  c\sqrt{d^3H^4K\cdot\log(dKH/\delta)}
\end{align*}
for some absolute constant $c>0$.
\end{proof}

\subsection{Proof of Lemma~\ref{lem:small_V}}
\begin{proof}[Proof of  Lemma~\ref{lem:small_V}]
We denote $\Delta^k = V_1^*(s^k_1, r^k) - \EE_{s\sim \mu}[V_1^*(s, r^k)]$.
Since $r^k$ depends only on data collected during the first $k-1$ episodes,
$\{\Delta^k \}_{k=1}^K$ is a martingale difference sequence.
Moreover, $|\Delta^k|\le H$ almost surely.
Thus, by Azuma-Hoeffding inequality, we have, with probability at least $1-\delta/8$, there exists an absolute constant $c_1 > 0$, such that
\[
\left|\sum_{k=1}^K\Delta^k\right|
\le c_1 H\sqrt{K\log(1/\delta)},
\]
which we condition on in the rest of the proof.
Therefore, we have,
\[
\EE_{s\sim \mu}\left[\sum_{k=1}^KV_1^*(s, r^k)\right]
\le \sum_{k=1}^KV_1^*(s, r^k)
 + c_1 H\sqrt{K\log(1/\delta)}.
\]
Next, we notice that for all $k\in[K]$, 
\[
\Lambda_h\succeq \Lambda_h^k.
\]
Hence we have for all $(k,h)\in [K]\times[H]$, 
\[
r^k_h(\cdot, \cdot) \ge u_h(\cdot, \cdot)/H.
\]
Hence\[
V_1^*(\cdot, u_h/H) \le 
V_1^*(\cdot, r_h^k).
\]
Together with Lemma~\ref{lem:sum_V}, we have
\begin{align*}
\EE_{s\sim \mu}\big[
V_1^*(s, u_h/H)
\big]
&\le 
\EE_{s\sim \mu}\left[
\sum_{k=1}^KV_1^*(s, r^k)/K
\right]
\le K^{-1}\sum_{k=1}^KV_1^*(s_1^k, r^k)
+ c_1 H\sqrt{\log(1/\delta)/K}\\
&\le c' \sqrt{d^3H^4\cdot\log(dKH/\delta)/K}
\end{align*}
for some absolute constant $c'>0$.
\end{proof}

\subsection{Proof of Lemma~\ref{lem:confidence_planning}}
\begin{proof}[Proof of Lemma~\ref{lem:confidence_planning}]
Using the same argument in the proof of Lemma~\ref{lem:sum_V}, with probability at least $1 - \delta / 4$, 
for all $h \in [H]$ and $(s, a) \in \states \times \actions$, 
we have
\[
\left|
\phi(s,a)^\top w_h - \sum_{s' \in \states} P_h(s' \mid s,a) V_{h+1}(s')
\right| \le \beta  \cdot \|\phi(s,a)\|_{(\Lambda_h)^{-1}}.
\]
Therefore, for all $h \in [H]$ and $(s, a) \in \states \times \actions$, 
\begin{align*}
&Q_h(s, a) \le (w_h)^\top \phi(s, a) + r_{h}(s, a) +  u_h(s, a)\\
 \le&  r_h(s, a) + \sum_{s' \in \states} P_h(s' \mid s,a) V_{h+1}(s') + 2\beta  \cdot \|\phi(s,a)\|_{(\Lambda_h)^{-1}}.
\end{align*}
Moreover, $Q_h(s, a) \le H$.
Since $u_h(\cdot, \cdot)=\min\left\{\beta\cdot \sqrt{\phi(\cdot,\cdot)^\top(\Lambda_h)^{-1}\phi(\cdot,\cdot)}, H\right\}$,
we have
\[
Q_h(s, a) 
\le
r_h(s, a) 
+ \sum_{s'}P_{h}(s' \mid s, a)V_{h+1}(s')
+
2u_h(s, a).
\]

Now we prove for all $h \in [H]$ and $(s, a) \in \states \times \actions$, $Q^*_h(s,a, r)\le Q_h(s, a)$.
We prove by induction on $h$. When $h = H + 1$ this is clearly true. 
Suppose for some $h \in [H]$, $Q^*_{h + 1}(s,a, r)\le Q_{h + 1}(s, a)$ for all $(s, a) \in \states \times \actions$.
We have
\[
Q_{h}(s, a) = \min\{(w_h)^\top \phi(s, a) + r_{h}(s, a) +  u_h(s, a), H\}.
\]
Since $Q^*_{h + 1}(s,a, r) \le H$ and $u_h(\cdot, \cdot)=\min\left\{\beta\cdot \sqrt{\phi(\cdot,\cdot)^\top(\Lambda_h)^{-1}\phi(\cdot,\cdot)}, H\right\}$, it suffices to prove that
\[
Q^*_{h + 1}(s,a, r) \le (w_h)^\top \phi(s, a) + r_{h}(s, a) +  \beta  \cdot \|\phi(s,a)\|_{(\Lambda_h)^{-1}}.
\]
By the induction hypothesis, 
\begin{align*}
\phi(s,a)^\top w_h 
\ge& \sum_{s' \in \states} P_h(s' \mid s,a) V_{h+1}(s')  - \beta  \cdot \|\phi(s,a)\|_{(\Lambda_h)^{-1}}\\
\ge& \sum_{s' \in \states} P_h(s' \mid s,a) V_{h+1}^*(s', r)  - \beta  \cdot \|\phi(s,a)\|_{(\Lambda_h)^{-1}}.
\end{align*}
Therefore,
 \begin{align*}
Q_h^*(s, a, r)& =  r_h(s, a) + \sum_{s' \in \states} P_h(s' \mid s,a) V_{h+1}^*(s', r) \\
& \ge   (w_h)^\top \phi(s, a) + r_{h}(s, a) +  \beta  \cdot \|\phi(s,a)\|_{(\Lambda_h)^{-1}} .
 \end{align*}
\end{proof}

\subsection{Proof of Theorem~\ref{thm:main}}
\begin{proof}[Proof of Theorem~\ref{thm:main}]
In our proof we condition on the events defined in Lemma~\ref{lem:small_V} and Lemma~\ref{lem:confidence_planning} which hold with probability at least $1 - \delta$.
By Lemma~\ref{lem:confidence_planning}, for any $s \in \states$, 
\[
V_1(s) = \max_{a \in \actions} Q_1(s, a) \ge \max_{a \in \actions} Q_1^*(s, a, r) = V_1^*(s, r),
\]
which implies
\[
\EE_{s_1\sim \mu}[V^{*}_1(s_1, r)
- 
V^{\pi}_1(s_1, r)]
\le 
\EE_{s_1\sim \mu}[V_1(s_1)
- 
V^{\pi}_1(s_1, r)]
.
\]
Note that
\begin{align*}
&\EE_{s_1\sim \mu}[V_1(s_1)- V^{\pi}_1(s_1, r)]\\
= & \EE_{s_1\sim \mu}[Q(s_1, \pi_1(s_1))- Q^{\pi}_1(s_1, \pi_1(s_1), r)]\\
= & \EE_{s_1\sim \mu, s_2 \sim P_1(\cdot \mid s_1, \pi_1(s_1))}[r_1(s_1, \pi_1(s_1)) + V_2(s_2) + u_1(s_1, \pi(s_1)) - r_1(s_1, \pi_1(s_1)) - V^{\pi}_2(s_2)]\\
= & \EE_{s_1\sim \mu, s_2 \sim P_1(\cdot \mid s_1, \pi_1(s_1))}[V_2(s_2) + u_1(s_1, \pi(s_1)) - V^{\pi}_2(s_2)]\\
= & \EE_{s_1\sim \mu, s_2 \sim P_1(\cdot \mid s_1, \pi_1(s_1)), s_3 \sim P_2(\cdot, \mid s_2, \pi_2(s_2))}[u_1(s_1, \pi(s_1)) + u_2(s_2, \pi(s_2)) +  V_3(s_3)- V^{\pi}_3(s_3)]\\
= & \ldots\\
=& \EE_{s\sim \mu}[V^{\pi}_1(s, u)].
\end{align*}
By definition of $V^{*}_1(s, u)$, we have
\[\EE_{s\sim \mu}[V^{\pi}_1(s, u)]
\le 
\EE_{s\sim \mu}[V^{*}_1(s, u)].\]
By Lemma~\ref{lem:small_V},
\[
\EE_{s\sim \mu}[V^{*}_1(s, u)] = H \cdot\EE_{s\sim \mu}[V^{*}_1(s, u / H)]\le c'H\sqrt{d^3H^4\cdot \log(dKH/\delta)/K}.
\]
By taking $K = c_K \cdot d^3H^6\log(dH  \delta^{-1}\varepsilon^{-1}) / \varepsilon^2$ for a sufficiently large constant $c_K > 0$, we have
\[
\EE_{s_1\sim \mu}[V^{*}_1(s_1, r)
- 
V^{\pi}_1(s_1, r)] \le H \cdot\EE_{s\sim \mu}[V^{*}_1(s, u / H)]\le c'H\sqrt{d^3H^4\cdot \log(dKH/\delta)/K} \le \varepsilon,
\]
which implies $\pi$ is $\varepsilon$-optimal with respect to $r$.
\end{proof} \section{Reward-Free RL under Linear $Q^*$ Assumption with a Simulator}
\label{sec:simulator}
In this section, we present an algorithm for reward-free RL under the linear $Q^*$ assumption (Assumption~\ref{asmp:lin_q_star}) in deterministic systems,
when the agent has access to a generative model (a.k.a. simulator) of the MDP.
More specifically, for each state action $(s, a) \in \states \times \actions$, for each $h \in [H]$, we assume the agent can query $P_h(s, a)$.
We show that after querying the transition operator for polynomial number of times during the exploration phase, during the planning phase, the agent can find an optimal policy for any given reward function $r$.

\begin{algorithm}[!t]
	\caption{Reward-Free RL under Linear $Q^*$: Exploration Phase}
	\label{algo:main_linear_q_star}
	\begin{algorithmic}[1]
		\For{$h=1, 2, \ldots, H$}
		\State Find $(s_h^1, a_h^1), (s_h^2, s_h^2), \ldots, (s_h^{d}, a_h^{d})$ such that $\phi(s_h^1, a_h^1), \phi(s_h^2, s_h^2), \ldots, \phi(s_h^{d}, a_h^{d})$ form a set of linear basis of $\mathrm{span}\left(\{\phi(s, a)\}_{(s, a) \in \states \times \actions}\right)$
		\For{$i = 1, 2, \ldots, d$}
		\State Query $t_h^i \gets P_h(s_h^i, a_h^i)$
		\EndFor
		\EndFor
		\State \Return $\mathcal{D}\gets\{(s^i_h, a^i_h, t^i_h)\}_{(i,h)\in [d]\times[H]}$
	\end{algorithmic}
	\label{algo:explore}
\end{algorithm}

\begin{algorithm}[!t]
	\caption{Reward-Free RL under Linear $Q^*$: Planning Phase}
\label{algo:batch_ls_linear_q_star}
	\begin{algorithmic}[1]
		\State \textbf{Input}: Dataset $\mathcal{D}=\{(s^i_h, a^i_h, t^i_h)\}_{(i,h)\in [d]\times[H]}$, reward functions $r = \{r_h\}_{h \in [H]}$
		\State $Q_{H+1}(\cdot,\cdot)\gets 0$ and $V_{H + 1}(\cdot) = 0$
		\For{step $h=H, H-1, \ldots, 1$}
		\For{$i = 1, 2, \ldots, d$}
		\State $Q_h(s^i_h, a^i_h) \gets r_h(s^i_h, a^i_h) + V_{h + 1}(t^i_h)$
		\EndFor
		\State $Q_h(s, a) \gets \sum_{i = 1}^d \beta_i \cdot Q_h(s^i_h, a^i_h)$ if $\phi(s, a) = \sum_{i = 1}^d \beta_i \cdot \phi(s^i_h, a^i_h)$
		\State $V_h(\cdot) = \max_{a \in \actions} Q_h(\cdot, a)$
		\State $\pi_h(\cdot) \gets \argmax_{a \in \actions} Q_h(\cdot, a)$
		\EndFor

		\State \textbf{Return} $\pi = \{\pi_h\}_{h \in [H]}$
	\end{algorithmic}
\end{algorithm}

The exploration phase of our algorithm is described in Algorithm~\ref{algo:main_linear_q_star}, while the planning phase is described in Algorithm~\ref{algo:batch_ls_linear_q_star}.

During the exploration phase, for each level $h$, we find $(s_h^1, a_h^1), (s_h^2, s_h^2), \ldots, (s_h^{d}, a_h^{d})$ such that \[\phi(s_h^1, a_h^1), \phi(s_h^2, s_h^2), \ldots, \phi(s_h^{d}, a_h^{d})\] form a set of linear basis of $\mathrm{span}\left(\{\phi(s, a)\}_{(s, a) \in \states \times \actions}\right)$ by querying the feature extractor $\phi$.
Then we query $P_h(s_h^i, a_h^i)$ for each $i \in [d]$.
During the planning phase, for each $(i, h) \in [d] \times [H]$, we calculate $Q_h(s_h^i, a_h^i) = r_h(s_h^i, a_h^i) + V_{h + 1}(P_h(s_h^i, a_h^i))$ by the Bellman equation.
For each $(s, a) \in \states \times \actions$, we can always find $\beta$ such that $\phi(s, a) = \sum_{i = 1}^d \beta_i \cdot \phi(s_h^i, a_h^i)$, since $\phi(s_h^1, a_h^1), \phi(s_h^2, s_h^2), \ldots, \phi(s_h^{d}, a_h^{d})$ form a set of linear basis of $\mathrm{span}\left(\{\phi(s, a)\}_{(s, a) \in \states \times \actions}\right)$.
Due to the linearity of the optimal $Q$-function, we set $Q_h(s, a) = \sum_{i = 1}^d \beta_i \cdot Q_h(s^i_h, a^i_h)$.
We define the $V$-function and the policy accordingly. 

Notice that during the exploration phase, the algorithm query the transition operator for $dH$ times in total.
To prove the correctness, we prove by induction on $h$ that during the planning phase, $Q_h(\cdot, \cdot) = Q^*_h(\cdot, \cdot)$.
Note that this is clearly true when $h = H + 1$.
Suppose $Q_{h + 1}(\cdot, \cdot) = Q^*_{h + 1}(\cdot, \cdot)$.
It is clear that $V_{h + 1}(\cdot) = V^*_{h + 1}(\cdot)$, which implies $Q_h(s^i_h, a^i_h) = Q_h^*(s^i_h, a^i_h)$ by the Bellman equation.
By Assumption~\ref{asmp:lin_q_star}, if $\phi(s, a) = \sum_{i = 1}^d \beta_i \cdot \phi(s^i_h, a^i_h)$, 
\[
Q_h(s, a) = \sum_{i = 1}^d \beta_i \cdot Q_h(s^i_h, a^i_h) = \sum_{i = 1}^d \beta_i \cdot Q_h^*(s^i_h, a^i_h)  = Q^*_h(s, a).
\] 

\section{Missing Proofs in Section~\ref{sec:hardness}}
In this hard instance construction in Section~\ref{sec:hardness}, for each $h \in [H - 2]$, for each $(s, a) \in \cS_h \times \actions$, we define $\phi(s, a) \in \mathbb{R}^{d}$ so that $\|\phi(s, a)\|_2 = 1$ and for any $(s', a') \in \cS_h \times \actions \setminus \{(s, a)\}$, we have $|\left(\phi(s, a)\right)^{\top} \phi(s', a')| \le 0.01$.
The following lemma demonstrates the existence of such feature extractor.  
\begin{lem}
There exists a set of vectors $\{\phi_1, \phi_2, \ldots, \phi_{2^H}\} \subset \mathbb{R}^d$ with $d = \mathrm{poly}(H)$ such that
\begin{enumerate}
\item $\|\phi_i\| = 1$ for all $i \in [2^H]$;
\item $|\phi_i^{\top}\phi_j| \le 0.01$ for all $i, j \in [2^H]$ with $i \neq j$.
\end{enumerate}
\end{lem}
\begin{proof}
This is a direct implication of Lemma A.1 in~\citep{Du2020Is} by setting $ n = 2^H$ and $\varepsilon = 0.01$.
\end{proof}
Note that the above lemma implies the existence of the required feature exactor, since for each $h \in [H - 2]$, there are less than $2^H$ state-action pairs in $\cS_h \times \actions$. 
We simply define the feature of the $i$-th state-action pair in $\cS_h \times \actions$ to be $\phi_i$ in the above lemma. 

\begin{proof}[Proof of Theorem~\ref{thm:lb}]
In order to prove Theorem~\ref{thm:lb}, by Yao's minimax principle~\citep{yao1977probabilistic}, it suffices to prove that for the hard distribution constructed in Section~\ref{sec:hardness}, for any deterministic algorithm $\mathcal{A}$ that samples at most $2^{H} / 100$ trajectories during the exploration phase, the probability (over the randomness of the hard distribution) that $\mathcal{A}$ outputs a $0.1$-optimal policy in the planning phase is at most $0.9$.

We first show that for the deterministic algorithm $\mathcal{A}$, among all the $2^{H - 2}$ choices for $(s_{H - 2}^*, a_{H - 2}^*)$, $s_{H - 1}^+$ is in the collected dataset $\mathcal{D}$ for at most $2^H / 100$ choices for $(s_{H - 2}^*, a_{H - 2}^*)$ during the exploration phase. 
Note that whenever $(s_{H - 2}, a_{H - 2}) \neq (s_{H - 2}^*, a_{H - 2}^*)$, we must have $s_{H - 1} = s_{H - 1}^-$ and $s_{H} = s_{H}^-$.
Therefore, the feedback received by $\mathcal{A}$ is always the same unless $(s_{H - 2}, a_{H - 2}) = (s_{H - 2}^*, a_{H - 2}^*)$.
However, since $\mathcal{A}$ samples at most $2^{H} / 100$ trajectories during the exploration phase, there are most $2^H / 100$ choices for $(s_{H - 2}^*, a_{H - 2}^*)$ during the exploration phase for which $s_{H - 1}^+$ is in the collected dataset $\mathcal{D}$.

Recall that $\mathcal{A}$ is deterministic.
For any choice of $(s_{H - 2}^*, a_{H - 2}^*)$, if $s_{H - 1}^+$ is not in the collected dataset $\mathcal{D}$, the collected dataset $\mathcal{D}$ is always the same, no matter $a_{H - 1}^* = 0$ or $a_{H - 1}^* = 1$.
Moreover, for any fixed choice of $(s_{H - 2}^*, a_{H - 2}^*)$, it can be verified that the reward function $r$ does not depend on the choice of $a_{H - 1}^*$.
Note that during the planning phase, algorithm $\mathcal{A}$ deterministically maps the collected dataset $\mathcal{D}$ and the reward function $r$ to a policy.
Furthermore, the only $0.1$-optimal policy must satisfy $\pi(s_h^*) = a_h^*$.
However, for any choice of $(s_{H - 2}^*, a_{H - 2}^*)$, if $s_{H - 1}^+$ is not in the collected dataset $\mathcal{D}$, $\pi(s_{H - 1}^*)$ does not depend on $a_{H - 1}^*$ since both the collected dataset $\mathcal{D}$ and the reward function $r$ do not depend on $a_{H - 1}^*$.
Therefore, for those choices of $(s_{H - 2}^*, a_{H - 2}^*)$, $\mathcal{A}$ outputs a $0.1$-optimal policy with probability at most $0.5$.
Therefore, the probability that $\mathcal{A}$ outputs a $0.1$-optimal policy is at most 
\[
\frac{2^H/100}{2^{H - 2}} + \left(1 - \frac{2^H/100}{2^{H - 2}}\right) / 2 \le 0.6.  
\]
\end{proof}

\end{document}